\documentclass[reqno, 11pt]{amsart}

\usepackage[style = numeric-comp,
            maxbibnames = 100,
            maxcitenames = 2,
            isbn = false,
            giveninits = true,
            backend = biber]{biblatex}
\usepackage{booktabs, array} 
\usepackage{rotating}
\usepackage{hyperref} 
\usepackage[section]{placeins} 
\usepackage{soul, color} 
\usepackage{scalerel} 
\usepackage{caption, subcaption} 
\usepackage{autobreak} 
\usepackage{cleveref} 
\usepackage{tikz} 
\usepackage{longtable} 
\usepackage{adjustbox}
\usepackage{bm} 
\usepackage{multicol}
\usepackage{physics}
\usepackage{cancel}
\usepackage{enumitem}
\usepackage{algorithm}
\usepackage{algpseudocode}
\usepackage{caption}
\usepackage{nameref}

\definecolor{darklightgrey}{RGB}{230, 230, 230}

\algrenewcommand\algorithmicrequire{\textbf{Input:}}
\algrenewcommand\algorithmicensure{\textbf{Output:}}

\usepackage{geometry} 
\allowdisplaybreaks 
\graphicspath{ {./images/} }
    \newcommand{\features}[1]{\ensuremath{\vb*{x_{k#1}}}}
    \newcommand{\catfeatures}[1]{\ensuremath{\vb*{b}_{\vb*{kj}#1}}}
    \newcommand{\compactcatfeatures}{\ensuremath{B_{k}}}
    
    \newcommand{\trainingresponse}{\ensuremath{y_{i}^{tr}}}
    \newcommand{\numtrainingfeatures}{\ensuremath{\vb*{x_{i}^{tr}}}}
    \newcommand{\cattrainingfeatures}{\ensuremath{\vb*{b_{ij}^{tr}}}}
    \newcommand{\compactcattrainingfeatures}{\ensuremath{B_{i}^{tr}}}
    \newcommand{\poisonfeatures}[1]{\ensuremath{\vb*{x_{k#1}^{p}}}}
    \newcommand{\compactpoisonfeatures}[1]{\ensuremath{{X_{#1}^{p}}}}
    \newcommand{\poisonresponse}{\ensuremath{\vb*{y_{k}^{p}}}}
    \newcommand{\catpoisonfeatures}[1]{\ensuremath{\vb*{b}_{\vb*{kj}#1}^{\vb*{p}}}}
    \newcommand{\compactcatpoisonfeatures}{\ensuremath{{B_{k}^{p}}}}
    \newcommand{\supercompactcatpoisonfeatures}{\ensuremath{{B^{p}}}}

\newcommand{\optimalregressionparam}{\ensuremath{\vb*{\theta}^{*}}}
\newcommand{\regressionparam}{\ensuremath{\vb*{\theta}}}
\newcommand{\weights}{\ensuremath{\vb*{w}}}
\newcommand{\numweights}{\ensuremath{\vb*{w^{n}}}}
\newcommand{\catweights}{\ensuremath{\vb*{w_{j}^{c}}}}
\newcommand{\bias}{\ensuremath{c}}
\newcommand{\regularization}{\ensuremath{\lambda \Omega(\vb*{w})}}

\newcommand{\poisonsamples}[1]{\ensuremath{\mathcal{D}^{p#1}}}

\newcommand{\numbertraining}{\ensuremath{n}}

\newcommand{\numberpoison}{\ensuremath{q}}

\newcommand{\numbernumfeatures}{\ensuremath{m}}

\newcommand{\numbercatfeatures}{\ensuremath{t}}

\newcommand{\numbercategories}{\ensuremath{n(j)}}

\newcommand{\suchthat}{\ensuremath{\text{s.t.}}}
\newcommand{\transpose}[1]{\ensuremath{{#1}^\top}}
\DeclareMathOperator*{\argmin}{arg\,min}
    \newtheorem{theorem}{Theorem}

\renewenvironment{abstract}
 {\small
  \begin{center}
  \bfseries \abstractname\vspace{-.5em}\vspace{0pt}
  \end{center}
  \list{}{%
    \setlength{\leftmargin}{5mm}
    \setlength{\rightmargin}{\leftmargin}%
  }%
  \item\relax}
 {\endlist}
 
\author{Monse Guedes-Ayala}
\author{Lars Schewe}
\author{Zeynep \c Suvak}
\author{Miguel Anjos}
\title[Poisoning Attacks against Ridge Regression Models]{Generating Poisoning Attacks against Ridge Regression Models with Categorical Features}

\begin{document}
\maketitle


\begin{abstract}
Machine Learning (ML) models have become a very powerful tool to extract information from large datasets and use it to make accurate predictions and automated decisions. 
However, ML models can be vulnerable to external attacks, causing them to underperform or deviate from their expected tasks.
One way to attack ML models is by injecting malicious data to mislead the algorithm during the training phase, which is referred to as a poisoning attack.
We can prepare for such situations by designing anticipated attacks, which are later used for creating and testing defence strategies. 
In this paper, we propose an algorithm to generate strong poisoning attacks for a ridge regression model containing both numerical and categorical features that explicitly models and poisons categorical features. 
We model categorical features as SOS-1 sets and formulate the problem of designing poisoning attacks as a bilevel optimization problem that is nonconvex mixed-integer in the upper-level and unconstrained convex quadratic in the lower-level.
We present the mathematical formulation of the problem, introduce a single-level reformulation based on the Karush-Kuhn-Tucker (KKT) conditions of the lower level, find bounds for the lower-level variables to accelerate solver performance, and propose a new algorithm to poison categorical features.  
Numerical experiments show that our method improves the mean squared error of all datasets compared to the previous benchmark in the literature.
\end{abstract}


\section{Introduction}
Machine learning (ML) is a branch of artificial intelligence (AI) that is concerned with developing algorithms that use statistical techniques to identify patterns in data and make predictions.
However, these algorithms are susceptible to adversarial attacks, which can significantly damage their performance.
An example of an attack ML models can be subject to is \textit{poisoning attacks}. 
In these attacks, a malicious actor introduces contaminated data points into the training dataset, causing the model to learn incorrect patterns and degrade its performance on unseen data.
An example of a poisoning attack is the generation of emails to mislead spam filters. 
We can introduce a small number of carefully crafted emails into the training data to cause the model to learn a biased decision boundary and overlook some spam emails. 
One common defence against adversarial attacks is a technique called adversarial training, which involves generating adversarial examples during the training process to make the model more robust to such attacks. 
By systematically designing and executing attack strategies, we can not only develop effective defence mechanisms but also gain valuable insights into the vulnerabilities of machine learning models to external manipulation. 
This paper focusses on developing poisoning attacks specifically targeting regression algorithms with categorical features using tools from mathematical optimization. 
The work presented here can be used to test the model's robustness and identify vulnerabilities the attackers might exploit. 

The main challenge for attackers lies in identifying the most effective data points to manipulate without being detected.
This process can be framed as an optimization problem, as the attacker seeks to maximise the impact on the machine learning model. 
Since the attacker optimises the attacks while taking into account the optimal response of the machine learning model to the new data, there is a hierarchical nature in this process. 
Because of this hierarchical nature, it is possible to frame it as a bilevel optimization problem.
Bilevel optimization covers a special class of mathematical optimization problems that model sequential and hierarchical decision-making between two or more players.
The first player optimises their decision by taking into account the optimal reaction of a second player.
This reaction will, in turn, influence the outcome of the first player. 
The first player is usually referred to as the upper level or leader, while the second player is called the lower level or follower.
Mathematically, the lower-level problem appears in the constraints of the upper level and only those decisions that are optimal for the follower are feasible for the leader. 
In the case of poisoning attacks, the leader is the attacker, who decides first on which data samples to inject. Then, the follower, the ML model, learns the optimal parameters using, among other samples, the samples injected by the attacker. 

The main focus of the research on poisoning attacks has been on classification algorithms, where the goal of the model is to predict the class to which data samples belong. 
The most employed strategy to solve this problem is to use gradient-based approaches to find optimal poisoning attacks. 
More recent papers have also considered poisoning attacks for regression models, where predictions consist of assigning continuous numerical values.
To the best of our knowledge, Jagielski et al. \cite{Jagielski2018} were the first to study poisoning attacks for linear regression models. 
They formulate the attack as a bilevel optimization problem and propose two attack strategies: an adaptation of a previously proposed gradient-based approach for classification models, and a statistical-based poisoning attack. 
Then, the papers \cite{Muller2020, Wen2021} aimed to improve the performance of \cite{Jagielski2018} in terms of the impact of the attack on the regression model and computational time. 
Li et al. \cite{Li2021} focus on poisoning attacks to modify regression parameters in a specific direction and solve the optimization problem by formulating its closed form. 
\c Suvak et al. \cite{Suvak2021} study in more depth the bilevel formulation of this problem and find the optimal poisoning samples by solving to local optimality the equivalent single-level reformulation of the problem.
Even though the authors of \cite{Suvak2021} discuss the possibility of considering categorical variables for this type of attack, they only solve the optimization problem for variables associated with numerical features. 
None of these approaches explicitly optimise categorical variables.
Instead, these variables are treated as numerical variables and then heuristically rounded to meet the binary structure of categorical features.
However, many datasets used for regression tasks involve a combination of numerical and categorical features \cite{Ng2019, Shyam-Prasad2022, guitierrez-gomez2020} and
categorical variables can be the key to making predictions.
One can think, for example, how much the predictions of house prices in London would change if the neighbourhood is not taken into account. 
Therefore, poisoning attacks that do not fully optimise categorical variables can be too weak for some applications. 

This paper introduces a novel mixed-integer bilevel optimization model for poisoning attacks in adversarial machine learning, specifically targeting cases with categorical features modelled as binary variables. 
To the best of our knowledge, this is the first application of mixed-integer bilevel optimization to poisoning attacks with categorical data.
Following this introduction, in \Cref{sec: mixed-integer formulation} we present the mixed-integer bilevel optimization model, followed by its single-level reformulation in \Cref{sec: KKT reformulation}.
In \Cref{sec: bounding bilinear}, we propose a method for bounding previously unbounded lower-level variables.  
This is followed by \Cref{sec: solution method}, which details a heuristic algorithm for generating poisoning attacks.
Computational experiments are presented in \Cref{sec: poisoning computational experiment}, testing on two publicly available datasets.
We conclude with some remarks and future research directions in \Cref{sec: poisoning conclusion}.
Main contributions include: the development of a mixed-integer bilevel optimization model for poisoning attacks involving categorical features, a method for bounding previously unbounded lower-level variables using a sensitivity analysis of the lower level, and a heuristic algorithm that generates stronger poisoning attacks than previous methods, demonstrating superior performance on benchmark datasets.

\section{ Mixed-integer bilevel formulation of poisoning attacks} \label{sec: mixed-integer formulation}
In this section, we present a mixed-integer bilevel formulation of the poisoning attacks problem. 
To the best of our knowledge, this is the first mixed-integer optimization model used to model categorical features in the design of poisoning attacks. 
We transform categorical variables using one-hot encoding and formulate the problem as a mixed-integer bilevel optimization problem. 
Recall that in our framework, the upper level (leader) is the attacker, who wants to poison the data used by the ridge regression model (follower) to make it perform as poorly as possible.
Then, the attacker's objective is to maximise the mean squared error (MSE) of the predictions made by the regression model for the unpoisoned data. 
The poisoned data injected by the leader is used in the regression algorithm, together with unpoisoned data, to fit the model and decide the optimal regression parameters.

Our formulation defines numerical and categorical features as variables.
Usually, when categorical features are used to fit regression models, they must first be transformed into numerical features. 
There are different ways to achieve this, but our focus will be on one-hot encoding.
This method creates a column (feature) for each category within a categorical feature and assigns 0 to all columns except to the category that each sample belongs to, which is assigned 1. 
When models with one-hot encoded categorical features are poisoned, the poisoning samples to be injected must also follow the structure just described. 
This means that for each categorical feature, only one of the features associated with each category can be equal to 1, while all others must be zero. 
The best way to model one-hot encoded features in our optimization framework is by including set-partitioning constraints, which are of the form $x_1 + x_2 + ... + x_n = 1$, where $x_i \in \{0,1\}$ are variables of the set we want to partition, in this case categories belonging to the same categorical feature. 
This is equivalent to special ordered sets of type 1 (SOS-1). 

To model this, we divide the sets of variables and data parameters into those belonging to numerical features and those belonging to categorical ones. 
We now define the notation for a given data sample indexed by $k$.
For the numerical features, define $\features{}$ as a $\numbernumfeatures$-dimension vector with all the values of numerical features, where $\numbernumfeatures$ is the number of numerical features of the data. 
Similarly, $\numweights$ is defined as the vector of variables of the weights associated with numerical features, which also has $\numbernumfeatures$ elements. 
For the categorical case, we must introduce a new index to denote the categorical feature a variable belongs to. 
Following this, we will have a collection of binary vectors of variables $\catfeatures{}{}$ of the form $\{0,1\}^{\numbercategories}$, where $j$ indicates the categorical feature, and $\numbercategories$ the number of categories that feature has.
We also define the vectors of variables for weights associated with categorical features as $\catweights$, where $j$ denotes the index of the categorical feature, and each of the vectors will have length $n(j)$. 
We denote with the superscript $p$ the samples that are poisoned (decision variables of upper-level), and with $tr$ the unpoisoned samples (data).

The whole set of data points to be poisoned can be expressed as $\poisonsamples{} = \{(\poisonfeatures{}, \compactcatpoisonfeatures{}, \poisonresponse)\}_{k=1}^{\numberpoison}$, where $\poisonfeatures{}$ is the vector of numerical features for sample $k$, $\compactcatpoisonfeatures{}$ is the set $\{\catpoisonfeatures{}\}_{j=1}^{\numbercatfeatures}$ of all the binary vectors associated with each categorical feature, and $\poisonresponse$ is the response variable. 
Even though the response variables $\poisonresponse$ can be optimized together with feature vectors, they are treated as a fixed parameter for the rest of our study. 
This is in line with the existing literature and allows us to focus on the poisoning of features and do a fair comparison in the computational experiments sections. 
We will discuss in \Cref{sec: poisoning computational experiment} how these parameters are initialised. 
The lower-level decision variables, that is, the parameters of the regression model, can be expressed as $\regressionparam = (\weights, \bias)$, where $\weights=(\numweights, \{\catweights\}_{j=1}^{t})$ are the weights, $t$ is the number of categorical features, and $\bias$ is the intercept. 
Once the linear regression parameters $\regressionparam$ are chosen, predictions for some sample $k$ are made by the following linear regression function:
\begin{align*}
    f \colon \mathbb{R}^{\numbernumfeatures} 
    \times \prod_{j=1}^{\numbercatfeatures} \{0,1\}^{\numbercategories} 
    \times \mathbb{R}^{\numbernumfeatures + \sum_{j=1}^{\numbercatfeatures} \numbercategories + 1} \to \mathbb{R},\\
    f(\features{}, \compactcatfeatures{} , \regressionparam) 
    = {\numweights}^{\top} \features{} + \sum_{j=1}^{\numbercatfeatures} {\catweights}^{\top} \catfeatures{} + \bias.
\end{align*}

After covering the main notation, we now present in more detail the two levels of the bilevel optimization formulation. 

\subsection{Upper level: attacker} \label{subsec: upper level attacker}
The objective of the upper level is to choose the values of the numerical features $\{\poisonfeatures{}\}_{k=1}^{q}$ and the categorical features $\{\compactcatpoisonfeatures{}\}_{k=1}^{q}$ of the poisoning samples that maximise the mean squared error.
In other words, they seek to maximise the distance between the response variables of training data points and the predictions made by the regression model. 
These predictions are obtained using the regression function $f(\features{}, \compactcatfeatures{}, \regressionparam)$ for each $k$, which depends on the data and the regression parameters $\regressionparam$ chosen by the lower level.
However, only those $\regressionparam$ that are obtained when fitting the regression model can be considered.
That is, only the optimal solutions of the lower-level problem are valid values for $\regressionparam$. 
The lower-level problem will be presented in \Cref{subsec: lower level follower}.
For now, we focus on the rest of the constraints of the upper level.
Then, two types of constraints need to be added to the attacker's problem. 
First, a set of constraints that ensures that all the variables $\catpoisonfeatures{}$ associated with categorical features are binary and that all numerical features are within the $[0,1]$ bounds. 
Second, set-partitioning constraints to make sure that for any sample $k$ and categorical feature $j$, all vectors $\catpoisonfeatures{}$ only contain one element equal to 1, while all others are forced to be 0. 
The upper-level problem is then:
\begin{align}
        \max_{\compactpoisonfeatures{}{}, \supercompactcatpoisonfeatures{}, \optimalregressionparam} \quad & \frac{1}{n} \sum_{i=1}^{n} \left( f(\numtrainingfeatures{}, \compactcattrainingfeatures{}, \optimalregressionparam) - \trainingresponse \right)^{2}  
         \label{leader objective}\\
        \suchthat \quad & \sum_{i=1}^{\numbercategories} \catpoisonfeatures{i} = 1, \quad k=1, \ldots ,\numberpoison, \quad j=1, \ldots ,\numbercatfeatures,  \label{SOS1 constraint}\\
        & \poisonfeatures{} \in [0,1]^{\numbernumfeatures}, \quad k=1, \ldots ,\numberpoison  \label{interval constraint},\\
        & \catpoisonfeatures{} \in \{0,1\}^{\numbercategories}, \quad k=1, \ldots ,\numberpoison, \quad j=1, \ldots ,\numbercatfeatures,  \label{binary constraint}\\
        & \optimalregressionparam \in \argmin \mathcal{L}(\compactpoisonfeatures{}, \supercompactcatpoisonfeatures{}),
\end{align}
where $\compactcattrainingfeatures{}$ is the set $\{\cattrainingfeatures{}{}\}_{j=1}^{\numbercatfeatures}$ of all the vectors associated with training categorical features for sample $i$, $\mathcal{L}(\compactpoisonfeatures{}, \supercompactcatpoisonfeatures{})$ is the lower-level problem, $\compactpoisonfeatures{}$ is the set of all vectors of poisoning samples of numerical features $\poisonfeatures{}$ for all samples $k$, and $\supercompactcatpoisonfeatures{}$ is the set of all $\compactcatpoisonfeatures{}$ for all samples $k$.
Since ridge regression is strictly convex, the lower level always has a unique solution and we do not need to decide between pessimistic or optimistic formulations. 
It is important to note that the variables of the lower level are the only ones appearing in the upper-level objective. 
Moreover, there are no coupling constraints, that is, upper-level constraints that depend on lower-level variables. 
By introducing binary variables, the upper-level problem becomes a mixed-integer quadratic problem.
We now turn to the formal formulation of the lower-level problem.

\subsection{Lower level: machine learning model} \label{subsec: lower level follower}
The lower level seeks to fit a ridge regression model given some training data, which is the combination of poisoned and unpoisoned data, to find optimal regression parameters $\regressionparam$. 
Its objective is to minimise the mean squared error between the response variable and the model predictions over a training set that includes the poisoning samples. 
A regularization term $\regularization$ is added to the objective, where $\lambda$ is a regularization parameter and $\Omega(\weights)=\|w\|_{2}^{2}$ uses the $l_{2}\text{-norm}$, which is differentiable over $\weights$. 
We distinguish between numerical weights, which are denoted by the vector $\numweights$, and the vectors of weights associated with categorical features $\catweights$, indexed by categorical feature $j$. 
The lower-level problem $\mathcal{L}(\compactpoisonfeatures{}, \supercompactcatpoisonfeatures{})$ can then be formulated as:
\begin{align}
        \min_{\regressionparam} \quad & \frac{1}{\numbertraining + \numberpoison} \bigg( \sum_{i=1}^{\numbertraining} \left( f(\numtrainingfeatures{}, \compactcattrainingfeatures{}, \regressionparam) - \trainingresponse \right)^{2} 
        + \sum_{i=1}^{\numberpoison} \left( f(\poisonfeatures{}, \compactcatpoisonfeatures{}, \regressionparam) - \poisonresponse \right)^{2} \bigg)  
        + \regularization ,
\end{align}
where $\regressionparam = (\numweights, \{\catweights\}_{j=1}^{\numbercatfeatures}, \bias) \in 
        \mathbb{R}^{\numbernumfeatures + \sum_{j=1}^{\numbercatfeatures} \numbercategories + 1}$.
When the upper-level variables are fixed, the lower-level problem is a convex unconstrained quadratic problem (QP). 
The resulting problem is a bilevel program whose upper level is a mixed-integer quadratic problem and whose lower level is an unconstrained quadratic problem with respect to the upper-level variables (MIQP-QP).
However, the upper level maximizes (instead of minimizing) a quadratic function.
When this is the case, we know that the optimal solution is at the boundary of the feasible region, where some of the constraints are active. 
However, since the feasible region of the upper level is non-convex and depends on the optimal solution of the lower level, this is a very challenging upper-level problem. 
The upper level has both binary and continuous variables, and both types of variables appear in the lower-level objective. 
The objective function of the upper level involves only lower-level variables, which are all continuous. 
The only constraints of the upper level are set-partitioning constraints that ensure that one-hot encoding is correctly modelled. 
These constraints only involve integer upper-level variables, which means that there are no coupling constraints. 
Similarly, the lower level is unconstrained, which implies that there are no linking variables (upper-level variables appearing on lower-level constraints). 
However, upper-level variables multiply lower-level variables in the lower-level objective, which can generate non-linearities in single-level reformulations. 
We study the single-level reformulation derived from the Karush-Kuhn-Tucker of the lower level in more detail in the following section. 

\section{Karush-Kuhn-Tucker reformulation}\label{sec: KKT reformulation}
Since the lower level is convex and quadratic, the first-order optimality conditions are sufficient for global optimality. 
Moreover, since it is an unconstrained problem, these conditions are just the partial derivative of the objective with respect to all lower-level variables. 
However, when the lower level is replaced by its KKT conditions, several bilinear constraints are added to the mixed-integer upper-level problem. 
The KKT optimality conditions of the lower level are:
\begin{align}
    \frac{\partial \mathcal{L}(\compactpoisonfeatures{}, \supercompactcatpoisonfeatures{})}{w^{n}_{m}} = 0 \quad \Rightarrow \quad
    & 
    \frac{2}{\numbertraining + \numberpoison}\bigg( \sum_{i=1}^{\numbertraining} \left( f(\numtrainingfeatures{}, \compactcattrainingfeatures{}, \regressionparam) - \trainingresponse \right)x_{im}^{tr} 
    + \sum_{k=1}^{\numberpoison} \left( f(\poisonfeatures{}, \compactcatpoisonfeatures{}, \regressionparam) - \poisonresponse \right)x_{km}^{p} \bigg)  \nonumber \\
    & + \lambda w^{n}_{m} = 0 \qquad 
     j=1, \ldots ,\numbernumfeatures  \label{KKT numerical weights}\\
    \frac{\partial \mathcal{L}(\compactpoisonfeatures{}, \supercompactcatpoisonfeatures{})}{w^{c}_{jz}} = 0 \quad \Rightarrow \quad 
    & \frac{2}{\numbertraining + \numberpoison}\bigg( \sum_{i=1}^{\numbertraining} \left( f(\numtrainingfeatures{}, \compactcattrainingfeatures{}, \regressionparam) - \trainingresponse \right)b_{ijz}^{tr} 
    + \sum_{k=1}^{\numberpoison} \left( f(\poisonfeatures{}, \compactcatpoisonfeatures{}, \regressionparam) - \poisonresponse \right)b_{kjz}^{p} \bigg) \nonumber \\
    & + \lambda w^{c}_{jz} = 0 \qquad  j=1, \ldots ,\numbercatfeatures \text{, } \label{KKT categorical weights} \ z=1, \ldots ,\numbercategories  \\
    \frac{\partial \mathcal{L}(\compactpoisonfeatures{}, \supercompactcatpoisonfeatures{})}{c} = 0 \quad \Rightarrow \quad
    & \frac{2}{\numbertraining + \numberpoison}\bigg( \sum_{i=1}^{\numbertraining} \left( f(\numtrainingfeatures{}, \compactcattrainingfeatures{}, \regressionparam) - \trainingresponse \right)
    + \sum_{k=1}^{\numberpoison} \left( f(\poisonfeatures{}, \compactcatpoisonfeatures{}, \regressionparam) - \poisonresponse \right)\bigg)  = 0  \label{KKT bias}
\end{align}
When we substitute the lower-level problem in the original formulation for these constraints, we get the following nonlinearly constrained problem with a quadratic objective: 
\begin{align} \label{single-level problem}
    \max & \quad \text{\eqref{leader objective}} \nonumber \\
    \suchthat & \quad \text{\eqref{SOS1 constraint}} - \text{\eqref{binary constraint}}, \\
    & \quad \text{\eqref{KKT numerical weights}} - \text{\eqref{KKT bias}}. \nonumber
\end{align} 
Note that the regression function $f$ applied to the poisoned data involves the multiplication of the weights and the poisoning samples, which are both variables of the single-level problem. 
Moreover, this function is multiplied again by the poisoning samples in some of the terms, leading to trilinear terms.
Those trilinear terms involving two binary variables can be fully linearised.
However, those trilinear terms with only one binary variable remain bilinear, and trilinear terms with only continuous variables can only be reformulated as bilinear by introducing new variables. 
It is important to find bounds for all the bilinear terms so that solvers can solve this model. 
We do this using sensitivity properties of the lower level and details can be found in \Cref{sec: bounding bilinear}. 
We use the KKT reformulation in our solution algorithm to optimise the continuous variables of the model.

\section{Bounding the bilinear terms} \label{sec: bounding bilinear} 
The tightness and quality of the convex relaxations that are used to model bilinear terms highly depend on the bounds of the variables. 
Moreover, most solution algorithms for bilevel and nonlinear optimization problems require a bounded feasible region. 
In this section, we present bounds for the lower-level variables that are obtained from the closed form reformulation of the lower level.
Even though the lower-level variables are unbounded in the original formulation of the problem, we show that it is possible to bound them by studying the sensitivity of ridge regression. 
The intuition behind this is that there is a limit on how much the poisoning samples can deviate the parameters of the regression model from those obtained when the model is fitted without poisoning samples.
We use the singular-value decomposition (SVD) and the closed form of the solution to the lower-level problem to bound all the lower-level variables in \Cref{thm: lower level bounds}. 
We do this by bounding the norm of the vector of variables by the largest eigenvalue of the unpoisoned part of the design matrix. 
We can separate the poisoned from the unpoisoned parts of the data by using Weyl's inequality. 
To the best of our knowledge, this is the first time that the sensitivity analysis of a lower-level model is used to bound the lower-level variables of a bilevel optimization problem. 

\begin{theorem}\label{thm: lower level bounds}
Consider the lower-level problem from \Cref{subsec: lower level follower}. Let $X_0$ be the submatrix of the design matrix belonging to the unpoisoned data, $\boldsymbol{y^0}$ the corresponding response variables, and $\regressionparam = (\numweights, \{\catweights\}_{j=1}^{\numbercatfeatures}, \bias) \in 
        \mathbb{R}^{\numbernumfeatures + \sum_{j=1}^{\numbercatfeatures} \numbercategories + 1}$ be lower-level variables.
Then, for all numerical features $i=1, \dots, m$, categorical features $j=1, \dots, t$, and category $z=1, \dots, n(j)$ we have the following bounds:
    \begin{equation}
        |w^n_i| \leq \frac{\sigma_t(X^0)}{\sigma_t^2(X^0) + \lambda} \left(\|\boldsymbol{y^0}\|_2 + \sqrt{q} + \frac{\sum_{i=1}^{n}y_i^0 + q}{\sqrt{n + q}}\right),
    \end{equation}
    \begin{equation}
        |w^c_{jz}| \leq \frac{\sigma_t(X^0)}{\sigma_t^2(X^0) + \lambda} \left(\|\boldsymbol{y^0}\|_2 + \sqrt{q} + \frac{\sum_{i=1}^{n}y_i^0 + q}{\sqrt{n + q}}\right),
    \end{equation}
    \begin{equation}
        \frac{\sum_{i=1}^{n}y_i^0}{n + q} \leq c \leq \frac{\sum_{i=1}^{n}y_i^0 + q}{n + q},
    \end{equation}
where $\sigma_t$ denotes the largest singular value of $X_0$, $\lambda$ is the regularization hyperparameter, $q$ is the number of poisoning samples and $n$ is the number of unpoisoned samples. 
\end{theorem}
\begin{proof}
Since the lower-level is convex, we can also express the optimal decision of the ridge regression model (lower-level) in closed form as $(\transpose{X} X + \lambda I)^{-1} \transpose{X} \boldsymbol{y}$. 
However, we have to be careful with bounding the intercept while not regularizing it. To do this, we can subtract the mean of the response variable from the target and set this mean to be the intercept. Following this, the closed form for the weights $\tilde{\regressionparam}$ and the intercept $c$ are
\begin{equation*}
    \tilde{\regressionparam} = \left(\transpose{X} X + \lambda I\right)^{-1} \transpose{X} \left(\boldsymbol{y} - \hat{y} \mathbf{1}_{n+q}\right), \qquad c = \hat{y},
\end{equation*}
where $\hat{y}$ is the mean of the response variables,$ \mathbf{1}_{n+q}$ is a vector of 1s of dimension $n+q$, where $n$ is the number of unpoisoned samples and $q$ is the number of poisoned samples. 
Furthermore, by defining the singular-value decomposition of $X$ as $X=US\transpose{V}$, we can rewrite the optimal regression weights as:
\begin{equation*}
    \tilde{\regressionparam} = V\left(\transpose{S} S + \lambda I \right)^{-1} \transpose{S} \transpose{U} \left(\boldsymbol{y} - \hat{y} \mathbf{1}_{n+q}\right).
\end{equation*}
We want to bound the vector of weights $\tilde{\regressionparam}$ from both sides, so we can bound its absolute value $|\tilde{\regressionparam}|$ by the absolute value of its elements $|\tilde{\regressionparam}_i|$.
Then, by the properties of norms and the orthonormality of $V$ and $U$ we have:
\begin{align*}
    |\tilde{\regressionparam}_{i}|  \leq \|\tilde{\regressionparam}\|_{2} 
    & \leq \|V(\transpose{S} S + \lambda I)^{-1} \transpose{S} \|_2 \|\transpose{U} (\boldsymbol{y} - \hat{y} \mathbf{1}_{n+q})\|_2 \\
    & = \|(\transpose{S} S + \lambda I)^{-1} \transpose{S} \|_2 \| (\boldsymbol{y} - \hat{y} \mathbf{1}_{n+q})\|_2
\end{align*}

Recall that the matrix $S$ was diagonal, and the inverse of a diagonal matrix is another diagonal matrix with the reciprocals of the original diagonal as diagonal. Following this, we have that $(\transpose{S} S + \lambda I)^{-1} \transpose{S} = \text{diag}((\sigma_i / (\sigma_i^2 + \lambda) \text{ for } i=1, \dots, t))$
where $\sigma_i$ are the singular values of $X$. 
By the definition of singular values, we have that $\sigma_i = \sqrt{\lambda_i(\transpose{X} X)}$ where $\lambda_i(\transpose{X} X)$ are the eigenvalues of $\transpose{X} X$ ordered from largest to smallest. 
Now, the norm of a diagonal matrix is equal to the largest element of the diagonal. 
We minimise the elements of this diagonal by picking the smallest eigenvalue of $\transpose{X}$. 
Therefore, we want to find an upper bound for this expression, which means that we should look for a lower-bound of $\lambda_t(\transpose{X} X)$. 

To bound the smallest eigenvalue of $\transpose{X} X$, we can use Weyl's inequality, which states that for $M = N + R$, where $N, R \in \mathbb{R}^{n\times n}$ are symmetric matrices with their respective eigenvalues $\mu_i, \nu_i, \rho_i$ ordered from largest to smallest for $i=1,\dots, n$,
the inequalities $\nu _{i}+\rho _{n}\leq \mu _{i}\leq \nu _{i}+\rho _{1}$ hold for all $i$. 
Moreover, if the matrix $R$ is positive semi-definite, then 
$\mu _{i} \geq \nu _{i}$ for all $i$.
We can apply this result to bound the smallest eigenvalue of the matrix $\transpose{X} X$, which depends both on the data and poisoning samples. To do this, we can write $\transpose{X} X$ as the sum of the matrices $\transpose{X^0} X^0$ and $\transpose{Z} Z$ where $X^0$ is the matrix with only data samples, and $Z$ is the matrix of poisoning samples:
\begin{equation}
    \transpose{X} X = \transpose{X^0} X^0 + \transpose{Z} Z
\end{equation}

Let $M = \transpose{X} X$, $N=\transpose{X^0} X^0$ and $R=\transpose{Z} Z$. Symmetric matrices of the form $\transpose{Z} Z$ are always positive semi-definite, which means that the smallest eigenvalue of $ \transpose{X} X$ is bigger than or equal to the smallest eigenvalue of $\transpose{X^0} X^0$. Therefore, we have that:
\begin{align*}
    \|(\transpose{S} S + \lambda I)^{-1} \transpose{S} \|_2 
    & \leq \frac{\sqrt{\lambda_t(\transpose{X^0} X^0)}}{\lambda_t(\transpose{X^0} X^0) + \lambda}
\end{align*}
Similarly, we can divide $\boldsymbol{y}$ into original data samples and poisoning samples as $\boldsymbol{y}^0, \boldsymbol{y}^p, \boldsymbol{y}$.
All $\boldsymbol{y}^p$ are taken by randomly picking a subset of $\boldsymbol{y}^0$, and that all elements of $\boldsymbol{y}^0$ are in $[0,1]$. Following this, we can define a bound in terms of the original data and the number of poisoning samples: 
\begin{equation}
    \|(\boldsymbol{y} - \hat{y} \mathbf{1}_{n+q})\|_2 \leq \|\boldsymbol{y^0}\|_2 + \sqrt{q} + \frac{\sum_{i=1}^{n}y_i^0 + q}{\sqrt{n + q}}
\end{equation}
where $q$ is the number of poisoning samples, and $n$ the number of non-poisoned samples. 
Getting everything together we get:
\begin{align*}
    |\tilde{\regressionparam}_{i}|  \leq 
    & \|(\transpose{S} S + \lambda I)^{-1} \transpose{S} \|_2 \| (\boldsymbol{y} - \hat{y} \mathbf{1}_{n+q})\|_2  \leq \frac{\sqrt{\lambda_t(\transpose{X^0} X^0)}}{\lambda_t(\transpose{X^0} X^0) + \lambda}  (\|\boldsymbol{y^0}\|_2 + \sqrt{q} + \frac{\sum_{i=1}^{n}y_i^0 + q}{\sqrt{n + q}})\\
\end{align*}
where $\hat{y}$ is also bounded by $\frac{\sum_{i=1}^{n}y_i^0}{n + q} \leq \hat{y} \leq \frac{\sum_{i=1}^{n}y_i^0 + q}{n + q}$. 

This completes the proof. 
\end{proof}
We use \Cref{thm: lower level bounds} to set upper and lower bounds for the lower-level variables, that is, the regression coefficients, when defining the variables of the model with an optimization solver. 
This procedure is also useful since it bounds the high point relaxation (HPR) of the bilevel problem.
The HPR is obtained by solving the upper-level problem without the optimality constraint of the lower level.
The boundness of the HPR is often a necessary assumption for many of the solution algorithms in the bilevel literature. 
Even though we have tailored these bounds for ridge regression, they can be extended to other convex machine learning models with closed-form solutions.
Future work could also make these bounds even tighter, integrating them into solution methods. 
However, this is out of the scope of this paper. 

\section{Solution method}\label{sec: solution method}
Commercial solvers such as Gurobi \cite{gurobi} can only return the optimal solution for a toy version of the single-level reformulation of the mixed-integer bilevel poisoning attack model with no more than two features of each type and only one poisoning sample. 
For slightly larger sizes, allowing up to 3 poisoning samples for the same number of features, Gurobi takes a long time to improve feasible solutions, and initial solutions obtained without a time limit are not better than simply optimising numerical features to local optimality. 
For even larger datasets, those solved by current benchmarks in the literature, it gets stuck at the initial LP. 
This shows that our MINLP problem is intractable with currently available solvers.
However, computational experiments on small datasets show that solving the mixed-integer non-linear programming (MINLP) problem and optimising categorical features leads to better results than solving just the non-linear programming (NLP) problem for numerical features. 
This motivates the need for some alternative method that allows us to optimise categorical features. 
Research interest in the study of mixed-integer bilevel problems has increased over the past decade since these problems are very common in real-life applications. 
However, most efficient solution algorithms are for the linear case and often rely on restrictive assumptions. 
The most widely used approach for solving this class of problems is branch-and-cut \cite{Fischetti2017, Tahernejad2020}.
When the lower level is non-linear but convex, it is often replaced by its optimality conditions.
Kleinert et al. \cite{OuterKleinert2021} propose an outer approximation algorithm to find the global optimum of mixed-integer quadratic bilevel problems with mixed-integer upper-level and continuous lower-level. 
In our case, we replace the unconstrained lower level with its KKT conditions. 
However, since upper-level decisions multiply lower-level variables in the lower-level objective, the reformulation of the lower-level problem will include multilinear terms,
making it unsuitable for this solution method.

In terms of solution algorithms for the non-linear case that do not involve single-level reformulations, most methods focus on converging bounds. 
Mitsos \cite{Mitsos2010} proposes a bounding algorithm to solve non-linear mixed-integer bilevel programs to global optimality. 
However, they are required to solve a MINLP to optimality at each step, meaning that they do not solve problems with more than three variables on each level. 
Merkert et al. \cite{Merkert2022} present an exact global solution algorithm for mixed-integer bilevel optimization with integer lower-level problems that can handle non-linear terms such as the product of upper- and lower-level variables. 
However, they require the lower-level variables to be integer, and in our case, they are all continuous. 
Avraamidou and Pistikopoulos \cite{Avraamidou2019} propose a solution algorithm to find exact global solutions of quadratic mixed-integer bilevel problems with bounded integer and continuous variables at both levels. 
Soares et al. \cite{Soares2021} present a bounding procedure for finding global optimality of mixed-integer bilevel problems. 
The problem with the algorithms from \cite{Soares2021, Avraamidou2019} is that they require solving the HPR with additional constraints at each iteration. 
Due to the structure of our lower-level problem, we would need to solve the maximisation of a convex function with non-linear constraints with bilinear terms at each iteration, which is not practical. 
Following this, we propose our own method to find feasible solutions to the mixed-integer bilevel problem.

In this section, we propose the use of a heuristic algorithm to get strong poisoning attacks.  
Our algorithm works by locally solving the continuous single-level reformulation of the bilevel problem with numerical features and iteratively \textit{flipping} categorical features in the direction that is most detrimental to the performance of the model. 
For the local optimization of the numerical features, we propose an enhanced version of the algorithm developed in \cite{Suvak2021}, which is presented in \Cref{subsec: SAS}.
The flipping procedure is covered in \Cref{ifcf}.

\subsection{Shifting attack strategy}\label{subsec: SAS}
In their paper, Şuvak at al. \cite{Suvak2021} tried optimising categorical features as continuous variables and then projecting them into binary, but they state that this approach did not show additional improvements. 
Following this, they decided to fix categorical features and focus on locally optimising numerical ones. 
To mitigate the possibility of getting stuck in local solutions, they develop an iterative attack algorithm. 
In this algorithm, poisoning attacks are iteratively added to the data, instead of all being optimised at the same time. 
They do this by partitioning the continuous upper-level variables into a $k$ disjoint subsets $ \compactpoisonfeatures{i}$ of equal size for $i=1, \dots, k$, where $\compactpoisonfeatures{} = \compactpoisonfeatures{1} \cup \dots \cup \compactpoisonfeatures{k}$.
They start solving the problem to optimise the subset of variables $\compactpoisonfeatures{1}$ for the full training data $D^{tr}$, which includes numerical $X^{tr}$ and categorical $B^{tr}$ features, as parameters.
Then, at each iteration, they increase the training data $D^{tr}$ by taking the union with the just poisoned subset.
After the first iteration, the new training data is $ D^{tr} = X^{tr} \cup B^{tr} \cup \compactpoisonfeatures{1} \cup B_1^p$, where $B_1^p$ are the categorical features (treated as parameters) associated with $\compactpoisonfeatures{1}$.
This allows the solver to potentially improve previously found optimal solutions. 
They refer to this algorithm as the \textit{ iterative attack strategy }(IAS).

The problem with this strategy is that it overlooks the fact that categorical features are fixed throughout the whole process. 
By leaving out samples that are not being poisoned, they ignore the impact that fixed categorical features might have on the regression parameters. 
We propose a new approach to overcome this, which we call \textit{shifting attack strategy} (SAS).
We also partition the poisoning samples into the same set of disjoint subsets $S = \{ \compactpoisonfeatures{i} \}_{i=1}^k$.
However, instead of optimising subsets of samples without taking into account future samples, we fix all poisoning attacks from the beginning and optimise batches of them by shifting along the data while all other samples are treated as parameters. 
This means that we let the training data be the original training data and all the poisoning samples that are not being poisoned at the moment. 
In other words, in the first iteration we solve problem \eqref{single-level problem} for variables $\compactpoisonfeatures{1}$, and parameters $D^{tr} = X^{tr} \cup B^{tr} \cup S\backslash \compactpoisonfeatures{1} \cup \bigcup_{T \in S\backslash \compactpoisonfeatures{1}} B^p_T$, where $B^p_T$ is the set of the categorical features associated with subset $T \subset S$.
An outline of this strategy can be found in \Cref{SAS algorithm}.

\begin{algorithm}
\caption{Shifting Attack Strategy (SAS)}
\begin{algorithmic}[1] \label{SAS algorithm}
\Require{Dataset, batch size of attacks $m$.}
\Ensure{Feasible solution $x^p$.}
\State Select subset of the training data to be poisoned;
\State Fix all numerical and categorical features;
\State Partition data into $n$ batches of size $m$;
\For{each batch}
\State Locally optimise selected samples and weights using a NLP solver;
\If{Previous objective is improved}
\State Update data with the new values;
\Else
\State Restore data to the old values;
\EndIf
\State Fix these new poisoned samples as parameters;
\EndFor
\State \Return Feasible solution $x^p$.
\end{algorithmic}
\end{algorithm}

\subsection{Iterative flipping of categorical features}\label{ifcf}
We embed the SAS algorithm in a heuristic algorithm to poison both numerical and categorical features which we call iterative flipping of categorical features (IFCF). 
The proposed algorithm iterates between poisoning numerical features with the SAS and updating categorical features. 
We propose a method for flipping categorical features so that they iteratively increase the objective value of the leader, in other words, the mean squared error of the predictions on training data. 
The proposed algorithm is as follows. 
We first update all the numerical features using the SAS algorithm while fixing all the categorical features.
This process will give us the regression coefficients $\regressionparam$, which can be broken down into the intercept $c$, a vector for the weights of numerical features $\numweights$ and a set of vectors $\{\catweights\}_{j=1}^t$. 
Then for each sample, we find the maximum and minimum weight for each categorical feature and define
\begin{equation*}
    U_{j} =  \max_i w_{ji}^c \ \textup{  and  } \ L_{j} = \min_i w_{ji}^c, \textup{  for  }j=1, \dots, t.
\end{equation*}
These are the weights that had the largest influence on the MSE.
We take those categorical features with the highest weights to \textit{push the predictions up} by making all those columns equal to 1 while keeping everything else 0.
Similarly, we take the smallest weights and \textit{push the prediction down} by making those columns 1 and the rest 0. 
After this, we calculate the error of the two new predictions for sample $k$: 
\begin{equation*}
    E_k^{\uparrow} = | \hat{y} - {\numweights}^{\top} \features{} + \sum_{j=1}^{\numbercatfeatures} U_j + \bias | \ \textup{  and  } \ 
    E_k^{\downarrow} = | \hat{y} - {\numweights}^{\top} \features{} + \sum_{j=1}^{\numbercatfeatures} L_j + \bias |.
\end{equation*}
We choose the one with the highest contribution to the increase in MSE. 
If none of the pushing routines increase the MSE, we leave them as they are. 
After this, we run ridge regression on the newly poisoned samples to update weights. 
This step is much cheaper than solving a mixed-integer or non-convex optimization problem. 
We iterate this process over all samples. 
After that, we again optimise numerical features by running SAS a second time. 
We repeat this whole process until there is no improvement in $E_k^{\uparrow}$ and $E_k^{\downarrow}$ for all $k$. 
Numerical experiments showed that this is achieved after just one iteration. 
The IFCF algorithm is summarised in \Cref{alg: IFCD}. 
\begin{algorithm}
\caption{Iterative Flipping of Categorical Features (IFCF)}
\begin{algorithmic}[1]\label{alg: IFCD}
\Require{Dataset, poisoning rate $r$, batch size of SAS, epochs of the whole process.}
\Ensure{Feasible solution $x^p$.}
\State optimise numerical features with SAS while keeping categorical features fixed;
\State Fix numerical features;
\For{Epoch in total number of epochs}
\For{Sample in poisoning samples}
\State Choose categories $U_{j}$ with the highest weights in each category;
\State Choose categories $L_{j}$ with the smallest weights in each category;
\State Make predictions with both subsets;
\State Pick the direction that maximizes $E_k$ (the distance between prediction and target).
\If{Previous objective is improved}
\State Update data with the new values;
\State Solve ridge regression to update weights. 
\Else
\State Restore it to the old values;
\EndIf
\EndFor 
\State Solve problem with SAS and NLP solver with the current data;
\If{Previous objective is improved}
\State Update data with the new values;
\Else
\State Restore it to the old values;
\EndIf
\EndFor
\State \Return Feasible solution $x^p$.
\end{algorithmic}
\end{algorithm}

\section{Computational experiments} \label{sec: poisoning computational experiment}

Following previous related works \cite{Suvak2021, Jagielski2018, Wen2021}, we perform all of our experiments on two publicly available regression datasets: \textit{House Price dataset} and \textit{Healthcare dataset}.
These datasets are available at \href{https://github.com/jagielski/manip-ml}{https://github.com/jagielski/manip-ml}.
We split the data into 300 training, 250 validation, and 500 testing samples, as in \cite{Suvak2021}.
The processed House Price dataset has 35 numerical and 19 categorical features, and the Healthcare dataset has 3 numerical and 38 categorical features. 
Each of the categorical features is converted into several categories when one-hot encoded.
However, we also consider some extra subsets of the datasets with fewer categorical features. 
For each dataset, we explore two extra datasets: one of all numerical and 5 categorical features, and one of all numerical and 10 categorical features.
Similarly to \cite{Suvak2021}, we consider the poisoning rates $r$ of 4\%, 8\%, 12\%, 16\%, and 20\%. 
We use \cite{Suvak2021} as a benchmark since they showed that modelling poisoning attacks as bilevel optimization problems produces stronger attacks than the previous state-of-the-art method of gradient descent for almost all cases. 
We initialise the poisoning samples as in \cite{Suvak2021, Jagielski2018}: for each poisoning rate, we randomly select $r$\% samples from the training data and set the response variable to be $1-y$, rounded to the nearest integer. 
As in \cite{Suvak2021}, we solve the single-level reformulation of the problem to local optimality using KNITRO (version 11) \cite{Byrd2006}. 
The ridge regularization parameter is chosen by applying grid search with 10-fold cross-validation on the unpoisoned data. 
All experiments are run on a MacBook Pro with 1.4 GHz Quad-Core Intel Core i5 and 8 GB RAM.

We start by comparing the performance of the IAS and the SAS for just the optimization of numerical features. 
The average results for 10 different train-test splits are shown in \Cref{tab: SAS vs IAS}, where $\Delta$ denotes the increase in MSE, defined as the percentage of improvement of our approach over the benchmark. 
Positive values of $\Delta$ denote an improvement in poisoning attacks with respect to the benchmark (increase of MSE), while negative values would indicate a worsening of the attack. 
In general, the SAS improves the MSE for both datasets and all numbers of categorical features.
These improvements go from between 0.01\% and 7.79\%. 
The improvements in the House Price dataset are similar for all subsets.
The smaller the poisoning rate, the bigger the improvement of SAS over IAS.
For the Healthcare dataset, the subset with 5 categorical features is the one with the largest improvements, reaching 0.73\% for the 4\% poisoning rate.
The difference in performance for the Healthcare dataset is expected to be smaller since this dataset has just three numerical features.
Since IAS and SAS algorithms are strategies for optimising numerical features, it makes sense that datasets with a large number of numerical features benefit more from the improvements of SAS. 
\begin{table}[!htbp]
\centering
\captionof{table}{Comparison of one run of the iterative attack strategy (IAS) and the shifting attack strategy (SAS).}\label{tab: SAS vs IAS}
\begin{adjustbox}{width=\textwidth}
\begin{tabular}{lrrrrrrrrrrrrr}
\toprule
& & \multicolumn{3}{c}{All numerical 5 categorical} && \multicolumn{3}{c}{All numerical 10 categorical} && \multicolumn{3}{c}{All numerical all categorical} \\
\cmidrule{3-5}\cmidrule{7-9}\cmidrule{11-13}
\rule{0pt}{10pt} 
Dataset & $r$(\%) & MSE$_{\textup{IAS}}$ & MSE$_{\textup{SAS}}$ &  $\Delta$\textbf{(\%)} && MSE$_{\textup{IAS}}$ & MSE$_{\textup{SAS}}$ &  $\Delta$\textbf{(\%)} && MSE$_{\textup{IAS}}$ & MSE$_{\textup{SAS}}$ &  $\Delta$\textbf{(\%)}\\
\midrule
House Price & 4 & 0.0109 & 0.0115 & 5.90 && 0.0106 & 0.0112 & 5.29 && 0.0107 & 0.0112 & 4.79 \\
& 8 & 0.0133 & 0.0144 & 7.79 && 0.0132 & 0.0140 & 6.60 && 0.0132 & 0.0140 & 5.98 \\
& 12 & 0.0168 & 0.0179 & 6.48 && 0.0168 & 0.0176 & 5.17 && 0.0168 & 0.0176 & 4.69 \\
& 16 & 0.0213 & 0.0222 & 4.37 && 0.0214 & 0.0221 & 3.34 && 0.0215 & 0.0221 & 3.03 \\
& 20 & 0.0268 & 0.0274 & 2.43 && 0.0269 & 0.0274 & 1.86 && 0.0270 & 0.0274 & 1.68 \\
\midrule
Healthcare & 4 & 0.0041 & 0.0041 & 0.73 && 0.0039 & 0.0039 & 0.06 && 0.0038 & 0.0038 & 0.03 \\
& 8 & 0.0072 & 0.0072 & 0.31 && 0.0071 & 0.0071 & 0.04 && 0.0070 & 0.0071 & 0.03 \\
& 12 & 0.0120 & 0.0120 & 0.10 && 0.0119 & 0.0120 & 0.03 && 0.0119 & 0.0119 & 0.02 \\
& 16 & 0.0181 & 0.0181 & 0.04 && 0.0181 & 0.0181 & 0.01 && 0.0180 & 0.0180 & 0.02 \\
& 20 & 0.0252 & 0.0252 & 0.02 && 0.0252 & 0.0252 & 0.01 && 0.0251 & 0.0251 & 0.01 \\
\bottomrule
\end{tabular}
\end{adjustbox}
\end{table}

The performance of the IAS and SAS algorithms highly depends on the regularization hyperparameter and the size of the batches that are iteratively optimised. \Cref{fig: hyperparameters} shows the MSE of the IAS and SAS algorithms for 5 different values for the hyperparameter $\lambda$ in the ridge regularization: 0.001, 0.01, 0.1, 1, and 10. 
As we can see, the larger the regularization parameter, the smaller the difference in the performance of the two methods. 
Moreover, this distance seems to be smaller for smaller poisoning rates. 
There is also a difference in performance depending on the size of the batches we optimise locally, which is also related to the hyperparameter. 
\Cref{fig: batch} shows the MSE for both algorithms and all poisoning rates for different batch sizes.
The figure on the left has the results for the hyperparameter chosen using cross-validation. 
For this hyperparameter, the batch size does not affect the performance of the algorithms. 
However, for a smaller regularization hyperparameter, the smaller the batch size the worse the performance of IAS compared to SAS.
Overall, we can see that we can achieve very good results using SAS with a batch size of half the samples, which is cheaper than smaller sizes.

\begin{figure}[!htbp]
    \centering
    \includegraphics[width=.6\linewidth]{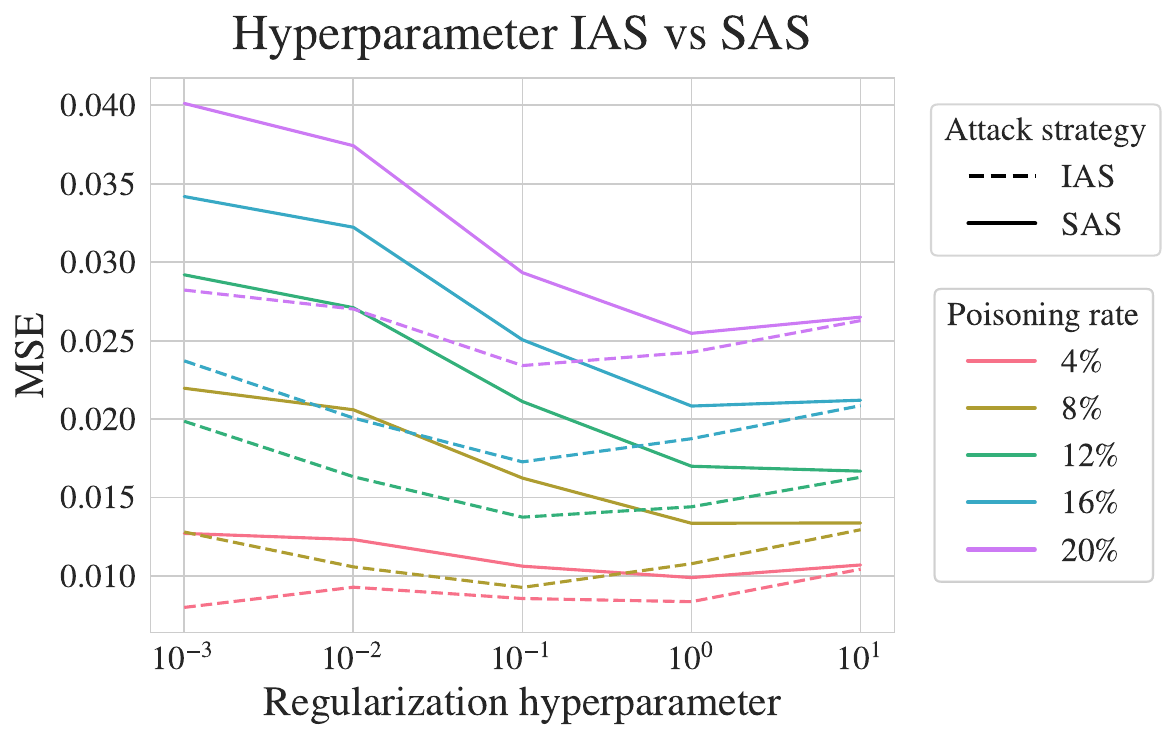}
    \caption{Difference between IAS and SAS for all poisoning rates and different regularization hyperparameters for the House Price dataset with all numerical and 5 categorical features and a 0.1 batch size.}
    \label{fig: hyperparameters}
\end{figure}
\begin{figure}[!htbp]
\centering
\begin{subfigure}{.45\textwidth}
  \centering
  \includegraphics[width=\linewidth]{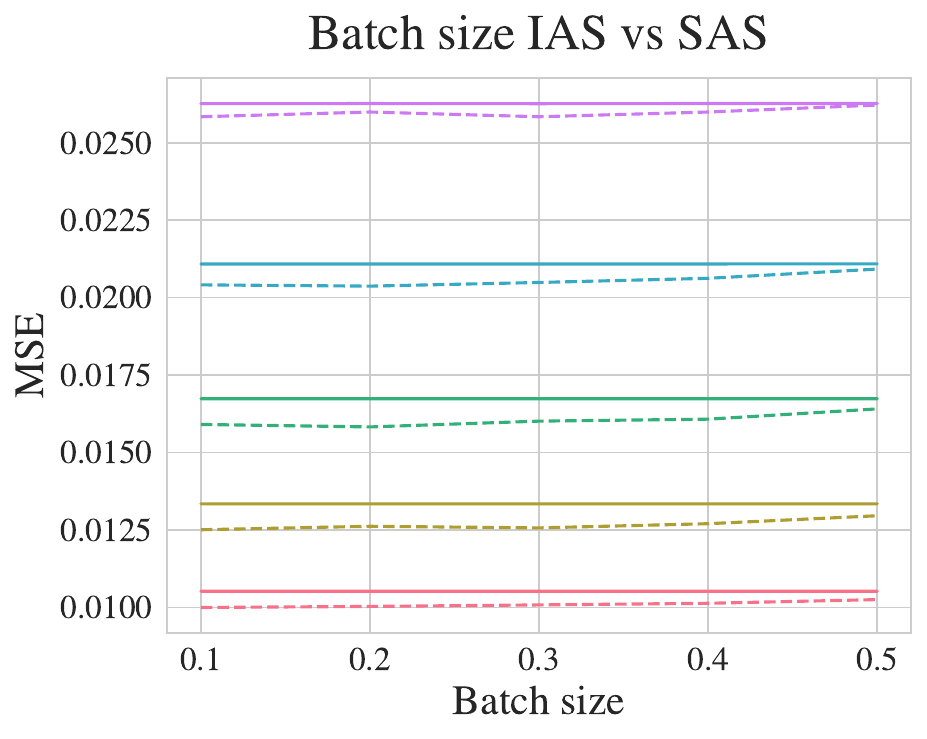}
  \caption{$\lambda$ = 4.211 (cross-validation)}
  \label{fig: batch cross}
\end{subfigure}%
\begin{subfigure}{.57\textwidth}
  \centering
  \includegraphics[width=\linewidth]{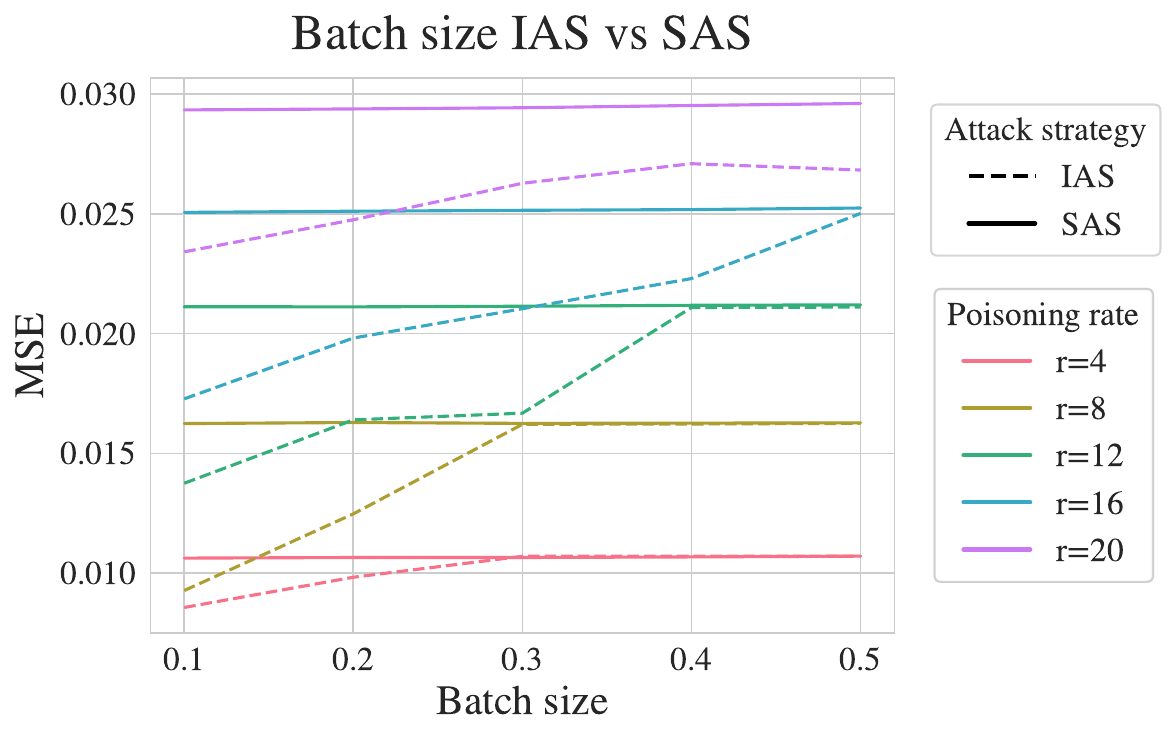}
  \caption{$\lambda$ = 0.1}
  \label{fig:batch 0.1}
\end{subfigure}
\caption{MSE of all poisoning rates for different batch sizes for the House Price dataset with all numerical and 5 categorical features.}
\label{fig: batch}
\end{figure}

We now compare the performance of the IFCF algorithms with the approach from \cite{Suvak2021}, where only numerical features are optimised using IAS.
From now on, we fix the batch size of IAS and SAS to 10\%.
We chose this value because it has the best performance for both methods. 
We run experiments for two types of datasets.
First, we have the training dataset, which is used when solving the optimization problem. 
Then, we test these poisoning attacks for some unseen testing data. 
We do this by taking the weights that are obtained when fitting the model to the poisoned data and then using them to make predictions for the unseen data.
Note that since this uses the same solutions, it does not make sense to present the computational time. 
\Cref{tab: allnum5cat} shows the results for all numerical and 5 categorical features. 
The IFCF algorithm improves the MSE for both training and testing data for the House Price dataset.
For training data, the improvements range from 4.55\% for the 20\% poisoning rate to 6.79\% for the 16\% rate.
In general, a higher poisoning rate leads to greater improvements. 
The MSE of the test data increases similarly, going from 2.30\% to 5.15\%.
The MSE of the healthcare dataset also increased with the IFCF, reaching an improvement of 5.01\% for the training dataset and 4.61\% for the testing data.
For this dataset, the larger the poisoning rate, the smaller the improvement. 
In terms of computational time, the benchmark is always faster than the IFCF algorithm.
This is as expected since the IFCF algorithm solves the NLP problem of the numerical features twice, while the benchmark does it only once. 
Moreover, it also does the flipping of categorical features. 
Still, the time taken to generate samples using the IFCF algorithm is less than double the time taken to solve the benchmark while the MSE of the IFCF is always larger.

\begin{table}[!htbp]
\centering
\captionof{table}{Average MSE and time for 10 runs of benchmark \cite{Suvak2021} and IFCF for all numerical and 5 categorical features.} \label{tab: allnum5cat}
\begin{tabular}{lrrrrrrrr}
\toprule
& & & \multicolumn{2}{c}{Suvak et al. \cite{Suvak2021}} && \multicolumn{2}{c}{IFCF} \\
\cmidrule{4-5}\cmidrule{7-8}
\multicolumn{1}{l}{Dataset} 
& \multicolumn{1}{c}{Type} 
& \multicolumn{1}{c}{$r$ (\%)} 
& \multicolumn{1}{c}{MSE} 
& \multicolumn{1}{c}{Time (s)} 
&
& \multicolumn{1}{c}{MSE} 
& \multicolumn{1}{c}{Time (s)} 
& \multicolumn{1}{c}{$\Delta$ (\%)}   \\ 
\midrule
House Price & Train & 4 & 0.011638 & 122.98 && 0.012172 & 206.24 &   4.59 \\
& &  8 & 0.015585 & 143.84 && 0.016358 & 252.26 &   4.95 \\
& & 12 & 0.020002 & 165.00 && 0.021167 & 297.93 &   5.82 \\
& & 16 & 0.024818 & 190.83 && 0.026502 & 346.56 &   6.79 \\
& & 20 & 0.030336 & 182.06 && 0.031715 & 351.08 &   4.55 \\
\cmidrule{2-9}
& Test & 4 & 0.012731 & - && 0.013024 & - &   2.30 \\
& &  8 & 0.016689 & - && 0.017241 & - &   3.31 \\
& & 12 & 0.020884 & - && 0.021674 & - &   3.78 \\
& & 16 & 0.025634 & - && 0.026955 & - &   5.15 \\
& & 20 & 0.031381 & - && 0.032358 & - &   3.11 \\
\midrule
Healthcare & Train & 4 & 0.003820 &  22.14 && 0.004011 &  35.13 &   5.01 \\
& &  8 & 0.007050 &  29.80 && 0.007281 &  46.75 &   3.28 \\
& & 12 & 0.011906 &  35.18 && 0.012135 &  56.76 &   1.92 \\
& & 16 & 0.018026 &  36.26 && 0.018249 &  67.83 &   1.23 \\
& & 20 & 0.025139 &  39.45 && 0.025352 &  66.21 &   0.85 \\
\cmidrule{2-9}
& Test & 4 & 0.004050 & - && 0.004237 & - &   4.61 \\
& &  8 & 0.007316 & - && 0.007535 & - &   2.99 \\
& & 12 & 0.012196 & - && 0.012421 & - &   1.84 \\
& & 16 & 0.018350 & - && 0.018566 & - &   1.18 \\
& & 20 & 0.025477 & - && 0.025695 & - &   0.86 \\
\bottomrule
\end{tabular}
\end{table}

When we increase the number of categorical features to 10, the results are very similar to those of the dataset with 5 features. 
Following this, we have decided to not include them in this part of the analysis.
The results of the full dataset with all the numerical and categorical features can be found in \Cref{tab: allnumallcat}.
For the House Price dataset, the increases in MSE go from 4.88\% for the 20\% poisoning rate to 8.81\% for the 8\% rate. 
Improvements of the 4\% rate for the testing data are on average 5.21\%, while for all larger poisoning rates, it is greater than 4.02\%. 
The impact of the IFCF algorithm on the Healthcare dataset is smaller. 
For training data, the largest improvement is for the 4\% poisoning rate with 4.17\%, while it is 3.56\% for the testing dataset.
The differences in computational time are as in the other two datasets since the main cost of the IFCF algorithm comes from running the NLP solver twice. 
When we decrease the tegularization hyperparameter, the performance of the IFCF is even better. 
\Cref{fig: house allnumallcat} shows all the runs for the House Price dataset with a hyperparameter $\lambda=0.1$, the average MSE among all the runs, and the geometric average of the increases in MSE.
For this hyperparameter, the improvements of the SAS algorithm over IAS are larger. 
The larger the poisoning rate, the bigger the improvement. 
The average improvement of the training data goes from 22\% for the 4\% poisoning rate to 36\% for the 20\% poisoning rate. 
For the testing data, improvements go from 11\% to 31\%.
We do not show results for the Healthcare dataset since the differences are smaller and harder to visualise, but similar patterns hold. 
Results for the datasets with 5 and 10 categorical features can be found in the appendix. 

\begin{table}[!htbp]
\centering
\captionof{table}{Average MSE and time for 10 runs of benchmark \cite{Suvak2021} and IFCF for all numerical and all categorical features.}
\begin{tabular}{lrrrrrrrrrrrrrrrrr}
\toprule
& & & \multicolumn{2}{c}{Suvak et al. \cite{Suvak2021}} && \multicolumn{2}{c}{IFCF} \\
\cmidrule{4-5}\cmidrule{7-8}
\rule{0pt}{10pt} 
Dataset & Type & $r$ (\%) & MSE & Time (s) && MSE & Time (s) & $\Delta$ (\%)   \\
    \midrule
House Price & Train & 4 & 0.010339 & 778.87 && 0.011049 & 1051.23 &   6.86 \\
& &  8 & 0.013643 & 822.50 && 0.014844 & 1155.86 &   8.81 \\
& & 12 & 0.017622 & 886.95 && 0.018890 & 1291.71 &   7.19 \\
& & 16 & 0.022180 & 922.82 && 0.023606 & 1367.21 &   6.43 \\
& & 20 & 0.027677 & 883.99 && 0.029028 & 1312.27 &   4.88 \\
\cmidrule{2-9}
& Test & 4 & 0.011844 & - && 0.012461 & - &   5.21 \\
& &  8 & 0.015358 & - && 0.016279 & - &   5.99 \\
& & 12 & 0.019055 & - && 0.020102 & - &   5.50 \\
& & 16 & 0.024098 & - && 0.025068 & - &   4.02 \\
& & 20 & 0.029026 & - && 0.030368 & - &   4.62 \\
\midrule
Healthcare & Train & 4 & 0.003832 & 1717.39 && 0.003992 & 2038.42 &   4.17 \\
& &  8 & 0.007066 & 1821.94 && 0.007256 & 2215.19 &   2.69 \\
& & 12 & 0.011915 & 1958.70 && 0.012116 & 2453.39 &   1.69 \\
& & 16 & 0.018011 & 2093.24 && 0.018236 & 2623.73 &   1.25 \\
& & 20 & 0.025114 & 2079.62 && 0.025337 & 2582.48 &   0.89 \\
\cmidrule{2-9}
& Test & 4 & 0.004070 & - && 0.004215 & - &   3.56 \\
& &  8 & 0.007316 & - && 0.007512 & - &   2.67 \\
& & 12 & 0.012172 & - && 0.012409 & - &   1.94 \\
& & 16 & 0.018282 & - && 0.018565 & - &   1.55 \\
& & 20 & 0.025388 & - && 0.025681 & - &   1.16 \\
\bottomrule
\end{tabular}
\label{tab: allnumallcat}
\end{table}

\begin{figure}[!htbp]
\centering
\begin{subfigure}{.5\textwidth}
  \centering
  \includegraphics[width=\linewidth]{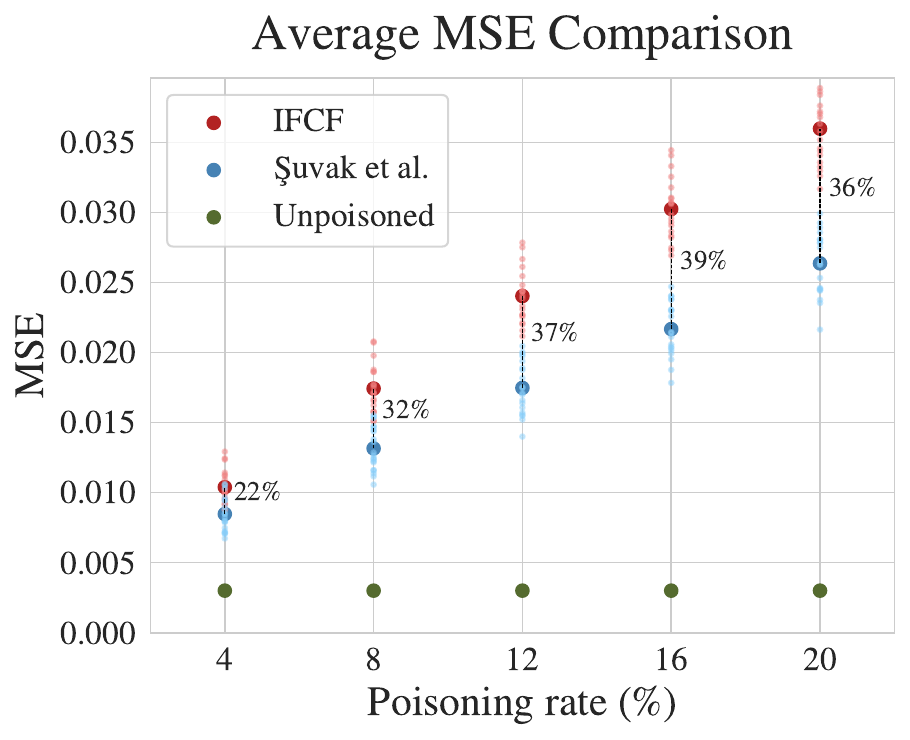}
  \caption{Train data}
\end{subfigure}%
\begin{subfigure}{.5\textwidth}
  \centering
  \includegraphics[width=\linewidth]{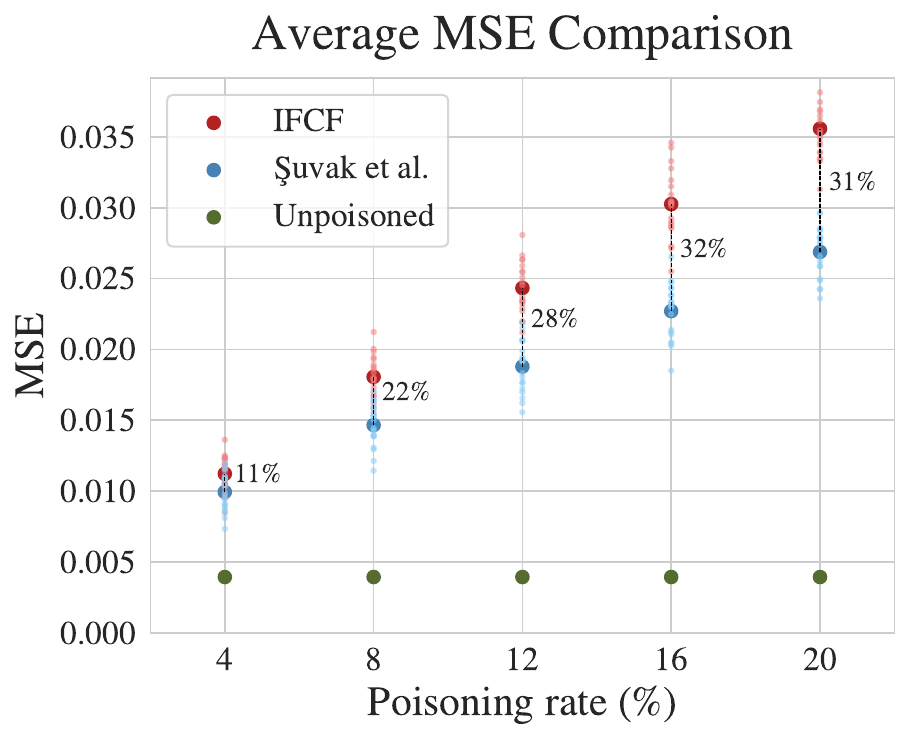}
  \caption{Test data}
\end{subfigure}
\caption{MSE of 20 runs and geometric mean of improvements for the House Price dataset with all numerical and all categorical features and $\lambda=0.1$.}
\label{fig: house allnumallcat}
\end{figure}

Overall, the larger the number of categorical features, the better the results are for the House Price dataset.
For the Healthcare dataset, they seem to be independent of the number of categorical features. 
Since the Healthcare dataset has only three numerical features, this suggests that the IFCF algorithm does not work so well with fewer numerical features, and it is a good method for a large dataset with many numerical and categorical features. 
Currently, poisoning data samples are not taken into account when finding the optimal hyperparameter $\lambda$. 
However, since this parameter has a very strong effect on the quality of the attacks, it would be interesting to explore in future research how poisoning attacks can be designed to also affect the hyperparameter selection. 
The proposed bilevel optimization approach demonstrates a significant improvement over benchmark techniques, highlighting the need for robust defences in adversarial machine learning involving categorical data. 
Our findings suggest that, by targeting categorical features specifically, poisoning attacks can achieve greater destabilization in certain machine learning models, paving the way for more targeted defences and testing protocols.
    
\section{Conclusion} \label{sec: poisoning conclusion}
Designing strong adversarial attacks is a useful tool to defend against potential attackers in adversarial machine learning, as it can help identify vulnerabilities, evaluate defences, and improve the robustness of machine learning models. 
In this paper, we have presented a novel approach for poisoning attacks of regression models in which categorical variables are modelled as binary decisions belonging to SOS sets of type 1. 
We frame these poisoning attacks as a mixed-integer bilevel optimization problem, which is later transformed into a single-level mixed-integer non-linear problem. 
A feasible solution for this problem that is better than previous attacks in the literature is found using an iterative algorithm that combines shifting along the numerical features to solve NLP problems to local optimality and a heuristic method to poison categorical features. 
We also present a method for bounding the lower-level variables, which allows us to bound the bilinear terms of the problem.
Computational experiments show that our algorithm improves the quality of the attack strategies over the benchmark for all the datasets considered. 
Moreover, our attacks generalise well to testing data.
The bilevel optimization approach developed for poisoning attacks in ridge regression models with categorical features, while novel, presents certain limitations. 
The use of heuristic methods for generating poisoning attacks means that solution optimality may not be guaranteed in all cases. 
We know that our attacks are stronger than existing attacks, but we do not know how close they are to the optimal solution to the mixed-integer model.
Exploring more robust exact approaches, along with scalability improvements, could help broaden the applicability of the proposed method to a wider range of adversarial machine learning scenarios.
The main challenge with exact approaches lies in the computational complexity introduced by the mixed-integer bilevel structure, especially as the number of categorical variables increases. 
An exact algorithm’s performance may degrade as data complexity rises, limiting its applicability to large-scale datasets or models with extensive categorical features. 
Another alternative is to extend our solution method to also include upper bounds that certify the quality of the attacks, but this is out of the scope of this paper. 
The model proposed in this paper can be used as the foundation for building and testing defence strategies. 
For example, we can have this problem as the two lower levels of a trilevel optimization problem. 
Here, the first level is the defence strategy, the second level is the attacker, and the last level is the ML model. 
It can also be used to study how sensitive and vulnerable ML models are to data changes in the direction that is most detrimental to the model. 
These extensions are suggested for the scope of future research.

\printbibliography

\newpage
\section*{Appendix}
\begin{figure}[!htbp]
\centering
\begin{subfigure}{.5\textwidth}
  \centering
  \includegraphics[width=\linewidth]{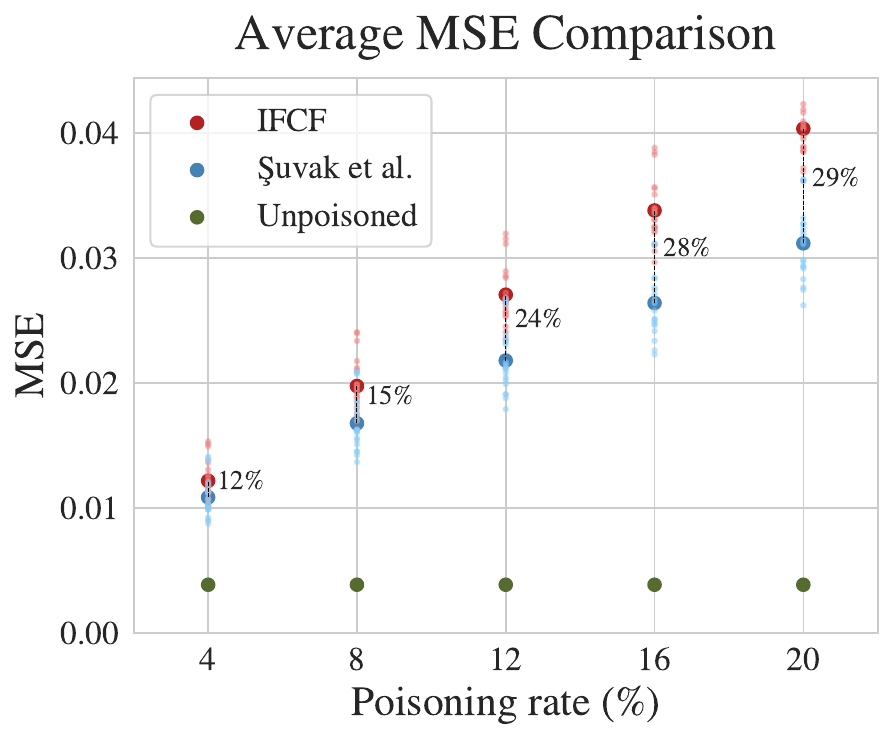}
  \caption{Training data}
\end{subfigure}%
\begin{subfigure}{.5\textwidth}
  \centering
  \includegraphics[width=\linewidth]{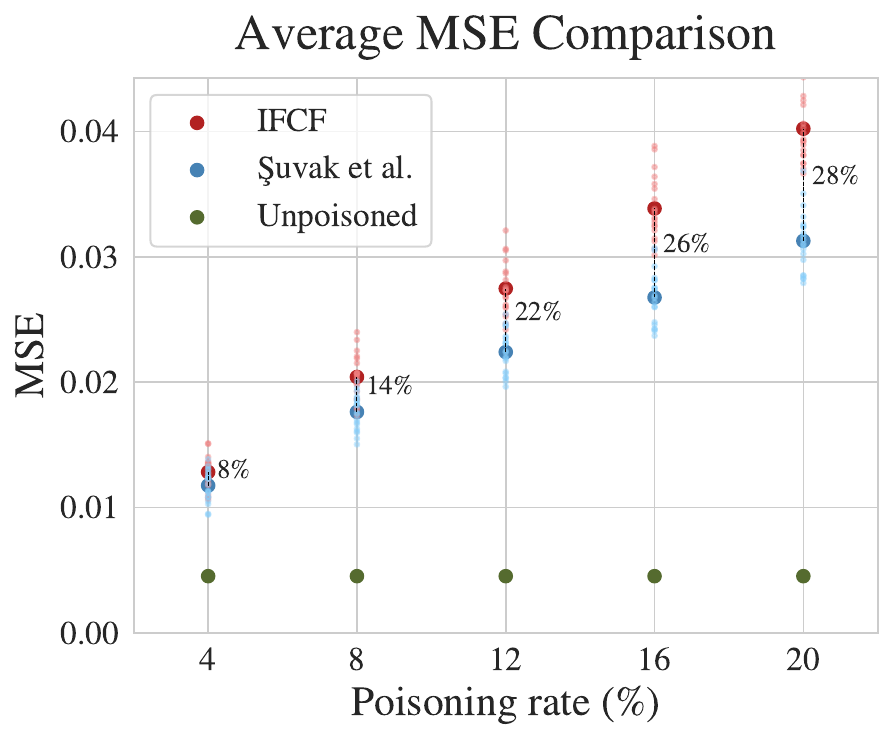}
  \caption{Testing data}
\end{subfigure}
\caption{MSE of 20 runs and geometric mean of improvements for the House Price dataset with all numerical and 5 categorical features and $\lambda = 0.1$.}
\label{fig: house allnum5cat}
\end{figure}

\begin{figure}[!htbp]
\centering
\begin{subfigure}{.5\textwidth}
  \centering
  \includegraphics[width=\linewidth]{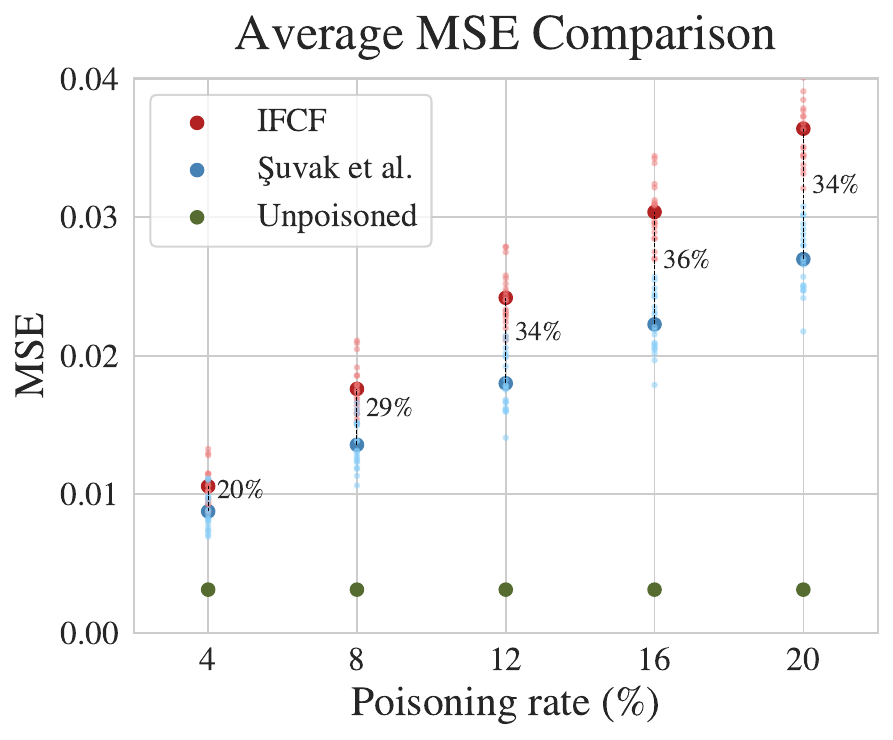}
  \caption{Training data}
\end{subfigure}%
\begin{subfigure}{.5\textwidth}
  \centering
  \includegraphics[width=\linewidth]{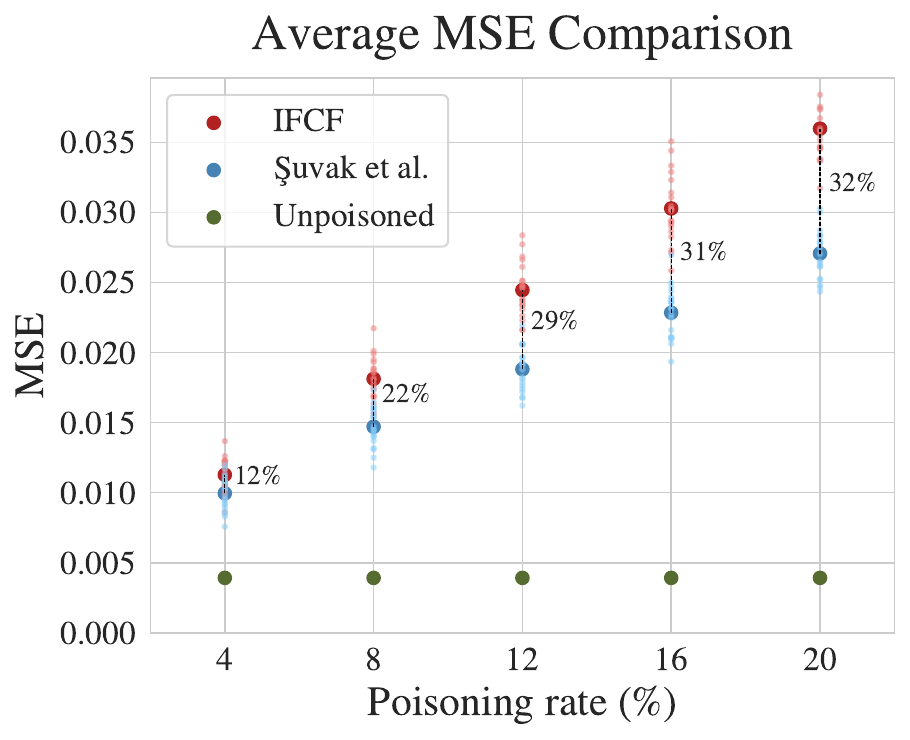}
  \caption{Testing data}
\end{subfigure}
\caption{MSE of 20 runs and geometric mean of improvements for the House Price dataset with all numerical and 10 categorical features and $\lambda=0.1$.}
\label{fig: house allnum10cat}
\end{figure}

\end{document}